\documentclass[12pt]{article}
\usepackage{epsfig,makeidx,amsmath,amsfonts,latexsym,subfigure}
\usepackage{times,pifont}

\def\smallpar{\smallbreak\@afterindentfalse\@afterheading\ignorespaces}

\newtheorem{theorem}{Theorem}[section]
\newtheorem{proposition}[theorem]{Proposition}
\newtheorem{lemma}[theorem]{Lemma}
\newtheorem{corollary}[theorem]{Corollary}
\newtheorem{definition}[theorem]{Definition}

\newenvironment{proof}{{\em Proof. }}{\hfill $\Box$ \vspace{1em}}

\newcommand{\E}{{\mathbb E}}

\newcommand{\G}{{\mathcal G}}

\newcommand{\Var}{{\mathbb V}{\rm ar}}

\newcommand{\Z}{{\mathbf Z}}
\newcommand{\W}{{\mathbf W}}
\newcommand{\K}{{\mathbf K}}
\newcommand{\z}{{\mathbf z}}
\newcommand{\w}{{\mathbf w}}
\newcommand{\kk}{{\mathbf k}}
\newcommand{\xx}{{\mathbf x}}

\newcommand{\Pro}{{\mathbb P}}

\newcommand{\taum}{\tau_{{\rm mix}}}

\newcommand{\n}{\|}
\newcommand{\bi}{\begin{itemize}}
\newcommand{\ei}{\end{itemize}}
\newcommand{\be}{\begin{enumerate}}
\newcommand{\ee}{\end{enumerate}}
\newcommand{\ra}{\rightarrow}

\newcommand{\ep}{\epsilon}

\newcommand{\iy}{\infty}

\newcommand{\beq}{\begin{equation}}
\newcommand{\eeq}{\end{equation}}
\newcommand{\beqa}{\begin{eqnarray*}}
\newcommand{\eeqa}{\end{eqnarray*}}
\newcommand{\btm}{\begin{theorem}}
\newcommand{\etm}{\end{theorem}}
\newcommand{\bpf}{\begin{proof}}
\newcommand{\epf}{\end{proof}}
\newcommand{\bla}{\begin{lemma}}
\newcommand{\ela}{\end{lemma}}
\newcommand{\bdn}{\begin{definition}}
\newcommand{\edn}{\end{definition}}
\newcommand{\bpn}{\begin{proposition}}
\newcommand{\epn}{\end{proposition}}
\newcommand{\bcy}{\begin{corollary}}
\newcommand{\ecy}{\end{corollary}}

\def\ldotsplus{\mathinner{\ldotp\ldotp\ldotp\ldotp}}
\def\fourdots{\relax\ifmmode\ldotsplus\else$\m@th \ldotsplus\,$\fi}

\begin{document}
\title{Fast mixing for Latent Dirichlet Allocation}
\author{Johan Jonasson\thanks{Chalmers University of Technology and University of Gothenburg}
\thanks{Research supported by the Knut and Alice Wallenberg Foundation, grant no.\ KAW 2012.0067}}
\maketitle
\begin{abstract}
Markov chain Monte Carlo (MCMC) algorithms are ubiquitous in probability theory in general and in machine learning in particular. A Markov chain is devised so that its stationary distribution is some probability distribution of interest. Then one samples from the given distribution by running the Markov chain for a "long time" until it appears to be stationary and then collects the sample. However these chains are often very complex and there are no theoretical guarantees that stationarity is actually reached. In this paper we study the Gibbs sampler of the posterior distribution of a very simple case of Latent Dirichlet Allocation, the arguably most well known Bayesian unsupervised learning model for text generation and text classification.
It is shown that when the corpus consists of two long documents of equal length $m$ and the vocabulary consists of only two different words, the mixing time is at most of order $m^2\log m$ (which corresponds to $m\log m$ rounds over the corpus). It will be apparent from our analysis that it seems very likely that the mixing time is not much worse in the more relevant case when the number of documents and the size of the vocabulary are also large as long as each word is represented a large number in each document, even though the computations involved may be intractable.

\end{abstract}

\noindent{\em AMS Subject classification :\/ 60J10 } \\
\noindent{\em Key words and phrases:\/  mixing time, MCMC, Gibbs sampler, conductivity, topic model}  \\
\noindent{\em Short title: Fast mixing for LDA}

\section{Introduction}

Markov chain Monte Carlo (MCMC) is a powerful tool for sampling from a given probability distribution on a very large state space, where direct sampling is difficult, in part because of the size of the state space and in part because of normalizing constants that are difficult to compute.

In machine learning in particular, MCMC algorithms are extremely common for sampling from posterior distributions of Bayesian probabilistic models. The posterior distribution given observed data is difficult to sample from for the given reasons. One may then design an (irreducible aperiodic) Markov chain whose stationary distribution is precisely the targeted posterior. This is usually fairly easy since the posterior is usually easy to compute up to the normalizing constant (the denominator in Bayes formula). One very often uses Gibbs sampling or the related Metropolis-Hastings algorithm.

Gibbs sampling in general can be described as follows. The state space is a finite set of random variables
$X=\{X_a\}_{a \in A}$, where $X_a \in T$ for some measurable space $T$, so that $X \in T^A$ and the targeted distribution is a given probability measure $\Pro$ on $T^A$. In order to sample from $\Pro$ one starts a Markov chain on $T^A$ whose updates are given by first choosing an index $a \in A$ at random and then choosing a new value of $X_a$ according to the conditional distribution of $X_a$ given all $X_b$, $b \in A \setminus \{a\}$. Under mild conditions, this Markov chain converges in distribution to $\Pro$.
The chain is then run for a "long time" (the "burn-in") whereupon a sample, hopefully approximately from $\Pro$, is collected. A key question here is for how long the chain actually has to be run, in order for the distribution after that time to be a good approximation of $\Pro$. Since $A$ is usually large, the number of steps needed should at least be no more than polynomial in the size of $A$ for Gibbs sampling to be feasible.
In almost all practical cases, the structure of the sample space and the probability measure $\Pro$ is so complex that is virtually impossible to make a rigorous analysis of the mixing rate. However it may be possible to consider some very simplified special cases.
In this paper, we will analyse a special case of Latent Dirichlet allocation, henceforth LDA for short, and demonstrate for such a simple special case, the mixing time is polynomial of low degree in the size of the problem.

LDA is a model used to classify documents according to their topics, which introduced by Blei et al \cite{BNJ}. One is faced with a large corpus of documents and wants to determine for each word in each document which topic it belongs to. Knowing this, one can then also classify the documents according to the proportion of words of the different topics it contains.
The setup in LDA is the following. The corpus consists of a fixed number $D$ of documents of lengths $N_d$, $d=1,\ldots,D$, a fixed set of topics $t_1,t_2,\ldots,t_s$ and a fixed set of distinct words $w_1,w_2,\ldots,w_v$. These are specified in advance. The number of topics is usually not large, whereas the number of distinct words is.
Next, for each document $d=1,\ldots,D$, independent multinomial distributions $\theta_d = (\theta_d(1),\ldots,\theta_d(s))$ over topics is chosen according to a Dirichlet prior with a known parameter
$\alpha=(\alpha_1,\ldots,\alpha_s)$. For each topic $t_i$, a multinomial distribution $\phi_i = (\phi_i(1),\ldots,\phi_i(v))$ according to a Dirichlet prior with known parameter $\beta=(\beta_1,\ldots,\beta_v)$ independently of each other and of the $\theta_d$:s.
Given these, the corpus is then generated by for each position $p=1,\ldots,N_d$ in each document $d$, picking a topic $z_{dp}$ according to $\theta_d$ and then picking the word $w_{dp}$ at that position according to $\phi_{z_{d,p}}$, doing this independently for all positions.
Note that the model is a so called "bag of words" model, i.e.\ it is invariant under permutations within each document.

In this paper, the very simple special case with $D=s=v=2$, $N_1=N_2=m$ and $\alpha=\beta=(1,1)$ will be studied with respect to the mixing time asymptotics as $m \ra \iy$. We will find that the corresponding Gibbs sampler indeed does mix fast (contrary to what was erroneously claimed in an earlier version of this paper). To simplify the notation, denote the two topics by $A$ and $B$ and the two words by $1$ and $2$. Define
$n_{ij}$ as the number of occurrences of the word $j$ in document $i$, $i,j = 1,2$ and write $n_{i.} = n_{i1}+n_{i2}$ (which by assumption equals $m$ in the case of study) and
$n_{.j} = n_{1j}+n_{2j}$ and $n_{..} = \sum_{i,j}n_{ij} = 2m$. We consider the mixing time for Gibbs sampling of the posterior in a seemingly typical case, namely that the number of $1$:s in the first document is $3m/10$ and in the second document $6m/10$. (Of course the precise numbers here are of no great significance as long as they are of order $m$.)
Let $R=\{R_t\}_{t=0}^\iy$ denote the corresponding Markov chain, whose state space is $\{A,B\}^{2m}$ and let $\pi_R$ denote the stationary distribution of $\{R_t\}$, i.e.\ the targeted posterior.

Before stating the main result, Theorem \ref{ta} below, the concept of mixing time needs to be introduced.
Let $X=\{X_t\}_{t=0}^\iy$ be a discrete time aperiodic irreducible Markov chain on the finite state space $S$ with transition matrix $[p(x,y)]_{x,y \in S}$. Let $\Pro_x$ be the underlying probability measure under $X_0=x$ and let $\pi$ be the stationary distribution.

\bdn
Let $\mu$ and $\nu$ be two probability measures on $S$. The the total variation distance between $\mu$ and $\nu$ is given by
\[\n \mu-\nu \n_{TV} = \max_{A \subset S}(\mu(A)-\nu(A)) = \frac12 \sum_{x \in S}|\mu(x)-\nu(x)|.\]
\edn

\bdn
For each $\kappa \in (0,1)$ and $x \in S$, the $(\kappa,x)$-mixing time of $X$ is given by
\[\taum(\kappa,x) = \taum^X(\kappa,x) = \min\{t:\n\Pro_x(X_t \in \cdot)-\pi\n_{TV} \leq \kappa\}.\]
\edn

\btm \label{ta}
Consider the case $n_{11} = 3m/10$ and $n_{21} = 6m/10$. Then there is a set $B \subseteq S$ with $\pi_R(B) = 1-m^{-10}$ such that for each $x \in B$ and each  $\kappa=\Omega(1/m)$ there is a $C(\kappa) < \iy$ such that
\[\taum^R(\kappa,x) \leq  C(\kappa)m^{2}\log m\]
for every $x \in B$.
\etm

The proof of Theorem \ref{ta} relies heavily on the {\em conductance} of a Markov chain.

\bdn
\bi
\item For $A \subseteq S$, the conductance is given as
\[Q(A,A^c) = \sum_{x \in A}\sum_{y \in A^c}\pi(x)p(x,y).\]
\item The conductance profile of $\{X_t\}$ is the decreasing function $\Phi:(0,\iy) \ra [0,1]$ given as follows. For $A \subset S$, let $\Phi_A = Q(A,A^c)/\pi(A)$ and for $r \leq 1/2$, set
\[\Phi(r) = \min\{\Phi_A:\pi(A) \leq r\}\]
and $\Phi(r)=\Phi(1/2)$ for $r > 1/2$.
\ei
\edn

The following is a consequence of Theorem 1 of \cite{MP}.
\begin{theorem} \label{tMP}
Let $\pi_* = \min\{\pi(x): x \in S\}$. Then for any $x \in S$,
\[\taum(\kappa,x) \leq 1+\int_{\pi_*}^{4/\kappa}\frac{4}{r\Phi(r)^2}dr.\]
\end{theorem}

Theorem \ref{tMP} is a refinement of earlier results on conductance bounds on mixing times. Write $\hat{\Phi} = \Phi(1/2)$. The following is from \cite{SJ,Sinclair2}.

\begin{theorem} \label{tSJ}
Assume that $X$ is reversible. Then
  \[\taum(\kappa,x)  \leq \hat{\Phi}^{-2} \left(\log(\pi(x)^{-1})+\log(\kappa^{-1})\right).\]
\end{theorem}

Note that any Gibbs sampler as defined above is reversible.

\smallskip

A powerful method for lower bounding conductance is by {\em canonical paths}, see \cite{JS,Sinclair1}.
Any reversible Markov chain can be seen as a weighted graph $\G$ by assigning weights $p(x,y)$ to each $e=\{x,y\}$ with $p(x,y)>0$.
For each pair of states $x,y$, choose a weighted path $\gamma_{xy}$ in $\G$ from $x$ to $y$. Even though the direction of the path is not important, it will later on be convenient in the arguments to keep the direction in mind.
Write $\Gamma$ for the collection $\{\gamma_{xy}: x,y \in S\}$. For each $e=\{x,y\}$, let $Q(e) = \pi(x)p(x,y)$ (which equals $\pi(y)p(y,x)$ under reversibility). The idea is to choose these paths in a way that makes the {\em path congestion} small.

\bdn
The path congestion of $\Gamma$ is given by
\[\rho(\Gamma) = \max_e\frac{1}{Q(e)} \sum_{x,y:e \in \gamma_{x,y}}\pi(x)\pi(y).\]
\edn

The key result relating path congestion to conductance is the following

\begin{lemma} \label{lpc}
For any collection $\Gamma$ of canonical paths,
\[\hat{\Phi} \geq \frac{1}{2\rho(\Gamma)}.\]
\end{lemma}

To finish off the proof, the classical method of {\em coupling} will be needed. Let $Y$ be a Markov chain with the same transition matrix as $X$ such that $Y_0$ is distributed according to $\pi$. Assume that the transitions of $X$ and $Y$ are made dependent, or {\em coupled}, in such a way that whenever $X_t=Y_t$ and $r>t$, then $X_r=Y_r$. Then the {\em coupling inequality} states the following.
\begin{lemma}
Let $T$ be the coupling time of $X$ and $Y$, i.e.\ $T=\min\{t:X_t=Y_t\}$ and assume that $X_0=x$. Then
\[\n \Pro(X_t \in \cdot)-\pi\n_{TV} \leq \Pro(T>t).\]
\end{lemma}

The rest of the paper is devoted to the proof of Theorem \ref{ta}.

\smallskip

{\bf Remarks on notation.} Let $\{a_n\}_{n=1}^\iy$ and $\{b_n\}_{n=1}^\iy$ be two sequences of positive numbers. We say that $a_n$ is at most of order $b_n$ and write $a_n = O(b_n)$ if there is a constant $C < \iy$ such that $a_n < Cb_n$ for all $n$. We say that $a_n$ is at least of order $b_n$ and write $a_n = \Omega(b_n)$ if $b_n=O(a_n)$. When both $a_n=O(a_n)$ and $a_n = \Omega(b_n)$, we say that $a_n$ is of order $b_n$ and write $a_n = \Theta(b_n)$.
When $a_n/b_n \ra 0$, we say that $a_n$ is of smaller order than $b_n$ and write $a_n=o(b_n)$.

We will also use the convenient shorthand notation $[n]$ for the set $\{1,2,\ldots,n\}$.

\smallskip

{\bf Other remarks.} Today there is a plethora of extensions of the LDA model; e.g.\ Andrews and Vigliocco \cite{AV} consider a hidden Markov model for the topics, Wallach \cite{Wallach} studies a hidden Markov model for the words and Gruber et el.\ \cite{GWR} consider a sentence based model. A good summary of the literature can be found in the introduction of Tian et el.\ \cite{TGHL}, who also study a sentence based model. For an introduction to probabilistic topic models, see \cite{Blei}.

\section{Proof of Theorem \ref{ta}}

The following two lemmas will be needed. The first one is a special case of a well known fact about moments of beta distributions (and more generally of Dirichlet distributions). The second one will only be used with $K=2$, but the general case comes at no extra cost.

\bla \label{la}
Let $X$ be a standard uniform random variable and $0 \leq k \leq n$. Then
\[\E[X^k(1-X)^{n-k}] = \frac{1}{(n+1){n \choose k}}.\]
\ela


\bla \label{lb}
For any nonnegative integers $a_{11},a_{12},\ldots,a_{1j},a_{21},a_{22},\ldots,a_{2K}$,
\[\binom{\sum_{i,j}a_{ij}}{\sum_{j}a_{1j}} \binom{\sum_{j}a_{1j}}{a_{11},\ldots,a_{1j}} \binom{\sum_{j}a_{2j}}{a_{21},\ldots,a_{2j}}\]
\[ = \binom{\sum_{i,j}a_{ij}}{a_{11}+a_{21},\ldots,a_{1K}+a_{2K}}\prod_{j}\binom{a_{1j}+a_{2j}}{a_{1j}},\]
where $i$ ranges over $[2]$ and $j$ over $[K]$.
\ela

\bpf
Both sides of the equality are equal to the multinomial coefficient
\[\binom{\sum_{i,j}a_{ij}}{a_{11},\ldots,a_{2K}}.\]
\epf

Let $\W = (w_{11},w_{12},\ldots,w_{1m},w_{21},\ldots,w_{2m})$ be the words in our corpus, let $\Z=(z_{11},z_{12},\ldots,z_{2m})$ be the latent topics and $\Z_d$ be the latent topics in document $d$. Let also $\theta_d$ be the probability that $z_{d1}=A$, $d=1,2$ and let $\phi_t$ be the conditional probability that $w_{dj}=1$ given that $z_{dj}=t$, $t=A,B$.
In the case under study, these four quantities are all independent standard uniform random variables.
We begin by determining the posterior distribution $\pi_R$ up to a normalizing constant.

Define $k_{dj}$ to be the number words in document $d$ with the topic being $A$ and the word being $j$ and the same dot notation for the $k$:s as for the $n$:s.

\begin{proposition} \label{piR}
  The posterior $\pi_R$ of LDA with two documents, two unique words, two topics and uniform priors is given by
  \[\pi_R(\z) = C\frac{{n_{..} \choose k_{..}}}{(k_{..}+1)(n_{..}-k_{..}+1){n_{1.} \choose k_{1.}}{n_{2.} \choose k_{2.}}{n_{.1} \choose k_{.1}}{n_{.2} \choose k_{.2}}},\]
  where $C$ is a normalizing constant.
\end{proposition}

\begin{proof}

By Bayes formula
\[\pi_R(\z) = \Pro(\Z=\z|\W=\w) \propto \Pro(\W=\w|\Z=\z)\Pro(\Z=\z).\]

Now
\beqa
\Pro(\Z = \z) &=& \E[\Pro(\Z=\z|\theta_1,\theta_2)] \\
&=& \E[\theta_1^{k_{1.}}(1-\theta_1)^{n_{1.}-k_{1.}}] \E[\theta_2^{k_{2.}}(1-\theta_2)^{n_{2.}-k_{2.}}]\\
&=& \frac{1}{(n_{1.}+1)(n_{2.}+1){n_{1.} \choose k_{1.}}{n_{2.} \choose k_{2.}}}.
\eeqa
where the last equality follows from Lemma \ref{la}.
For the second factor we have analogously, again using Lemma \ref{la},

\begin{align*}
\Pro(\W=\w|\Z=\z) &= \E[\phi_A^{k_{.1}}(1-\phi_A)^{k_{..}-k_{.1}}]\E[\phi_B^{n_{.1}-k_{.1}}(1-\phi_B)^{n_{..}-(n_{.1}-k_{.1})}] \\
&= \frac{1}{(k_{..}+1)(n_{..}-k_{..}+1){k_{..} \choose k_{.1}}{n_{..}-k_{..} \choose n_{.1}-k_{.1}}}.
\end{align*}

Hence, ignoring factors that do not depend on the $k$:s, using Lemma \ref{lb} with $K=2$, $a_{1j}=k_{.j}$ and $a_{2j}=n_{.j}-k_{.j}$ for the second equality and again ignoring a factor that does not depend on the $k$:s,
\begin{align*}
\pi_R(\z) &\propto \left( (k_{..}+1)(n_{..}-k_{..}+1){n_{1.} \choose k_{1.}}{n_{2.} \choose k_{2.}}{k_{..} \choose k_{.1}}{n_{..}-k_{..} \choose n_{.1}-k_{.1}}\right)^{-1} \\
&= \frac{{n_{..} \choose k_{..}}}{(k_{..}+1)(n_{..}-k_{..}+1){n_{1.} \choose k_{1.}}{n_{2.} \choose k_{2.}}{n_{.1} \choose k_{.1}}{n_{.2} \choose k_{.2}}}.
\end{align*}

\end{proof}

The expression for $\pi_R$ only depends on $\z$ via $\kk = \kk(\z) := (k_{11},k_{12},k_{21},k_{22})$. Identifying all $\z$ having the same $\kk(\z)$, we have (regarding,with some abuse of notation, a $\kk$ also as the equivalence class consisting of all $\z$:s having that particular $k_{dj}$:s) that for any $\kk_1$ and $\kk_2$, all $\z \in \kk_1$ have the same probability of transitioning into $\kk_2$.
Hence the process where we only record the $\kk$:s is a lumped Markov chain, whose state space is $[n_{11}] \times [n_{12}] \times [n_{21}] \times [n_{22}]$.
Denote this lumped Markov chain by $L=\{L_t\}_{t=0}^\iy$. By summing $\pi_R(\z)$ over the $\z$:s having the same value of $\kk(z)$, gives the following.

\begin{proposition} \label{piL}
  The lumped Gibbs sampler $L$ has the stationary distribution
  \[\pi_L(\kk) = \Pro(\K=\kk|\W=\w) = C\frac{{n_{..} \choose k_{..}}{n_{11} \choose k_{11}}{n_{12} \choose k_{12}}{n_{21} \choose k_{21}}{n_{22} \choose k_{22}}}{(k_{..}+1)(n_{..}-k_{..}+1){n_{1.} \choose k_{1.}}{n_{2.} \choose k_{2.}}{n_{.1} \choose k_{.1}}{n_{.2} \choose k_{.2}}}\]
  where $C$ is a normalizing constant and $\K=\kk(Z)$.
\end{proposition}
The most effort will go into analyzing the mixing time of $L$.

\smallskip

Define
\[h(x) = \left( x^x(1-x)^{1-x} \right)^{-1},\, x \in (0,1).\]
Then, for $x \in [0,1]$, by Stirling's formula,
\[{m \choose xm} = C\frac{h(x)^m}{\sqrt{m}\sqrt{(x+1/m)(1-x+1/m))}},\]
where $C$ is of constant order.
Hence, by cancelling factors independent of $(a,b,c,d)$ and writing $s(x) = \sqrt{(x+1/m)(1-x+1/m)}$, we get
\[
\pi_L(am,bm,cm,dm) = CG(a,b,c,d)^m \frac{s(\frac{10(a+c)}{9})s(\frac{10(b+d)}{11})s(a+b)s(c+d)}{s(\frac{a+b+c+d}{2})^3s(\frac{10a}{3})s(\frac{10b}{7})s(\frac{5c}{3})s(\frac{5d}{2})},
\]
$(a,b,c,d) \in [0,3/10] \times [0,7/10] \times [0,3/5] \times [0,2/5]$, where $C$ is of constant order and
\begin{equation} \label{Gdef}
G(a,b,c,d) = \left(\frac{h(\frac{(b+d-a-c)}{2})^2 h(\frac{10a}{3})^{3/10} h(\frac{10b}{7})^{7/10} h(\frac{5c}{3})^{3/5} h(\frac{5d}{2})^{2/5}}{h(\frac{10(a+c)}{9})^{9/10} h(\frac{10(b+d)}{11})^{11/10} h(a+b) h(c+d)}\right)^m.
\end{equation}

We will now analyze the function $G$. Define $g:=\log G$. Intuitively in order to get $3m/10$ ones in the first document and $3m/5$ ones in the second document, one has certain restrictions to the $\phi$'s and $\theta$'s. Working this out, one should expect to have $(a,b,c,d)$ very close to $\kk(u,v) = (k_1(u,v),k_2(u,v),k_3(u,v),k_4(u,v))$, $u \in [0,3/10]$, $v \in [3/5,1]$, where
\begin{align}
k_1(u,v) &= \frac{u(v-3/10)}{v-u} & k_2(u,v) = \frac{(1-u)(v-3/10)}{v-u} \label{k1k2} \\
k_3(u,v) &= \frac{u(v-3/5)}{v-u} & k_4(u,v) = \frac{(1-u)(v-3/5)}{v-u}, \label{k3k4}
\end{align}
i.e.\ $G$ should be maximal around the surface
\[\Pi := \{\kk(u,v): (u,v) \in [0,3/10] \times [3/5,1]\}\]
in $[0,3/10] \times [0,7/10] \times [0,3/5] \times [0,2/5]$.
To make this more precise, we first claim the following.

\begin{proposition}
  The function $G$ as defined above is constant on $\Pi$.
\end{proposition}

\bpf
The claim is equivalent to saying that all four partial derivatives of $g$ are zero at each $(a,b,c,d) \in \Pi$. It is helpful to spell out $g$:

\begin{align}
  g(a,b,c,d) &= 2\log h\left(\frac{a+b+c+d}{2}\right) \nonumber \\
   &+ \frac{3}{10} \log h\left(\frac{10}{3}a\right) + \frac{7}{10} \log h\left(\frac{10}{7}b\right) \nonumber \\
   &+ \frac{3}{5} \log h\left(\frac{5}{3}c\right) + \frac{5}{2} \log h\left(\frac{2}{5}d\right) \nonumber \\
  &- \frac{9}{10} \log h\left(\frac{10(a+c)}{9}\right) - \frac{11}{10} \log h\left(\frac{10(b+d)}{11}\right) \nonumber \\
  &- \log h(a+b) - \log h(c+d). \label{gexpr}
\end{align}

Hence, on observing that $\frac{d}{dx}h(x) = \log((1-x)/x)$,
\begin{equation} \label{gder}
  g_a'(a,b,c,d) = \log \frac{(1-\frac{a+b+c+d}{2})(1-\frac{10a}{3})(\frac{10(a+c)}{9})(a+b)}{(\frac{a+b+c+d}{2})(\frac{10a}{3})(1-\frac{10(a+c)}{9})(1-a-b)}
\end{equation}

For $(a,b,c,d) \in \Pi$, substitute with (\ref{k1k2}) and (\ref{k3k4}) and get, after some algebra
\[g_a'(a,b,c,d) = \log \frac{(\frac{9}{20}-u)(v-\frac{10}{3}uv)(\frac{20}{9}uv-u)(v-\frac{3}{10})}
{(v-\frac{9}{20})(\frac{10}{3}uv-u)(v-\frac{20}{9}uv)(\frac{3}{10}-u)}\]
which after some further algebra is seen to be $0$ for all $(u,v)$. Analogously, one finds that the three other partial derivatives vanish on $\Pi$. Since $g$ is clearly differentiable in the interior of its domain, this shows that $g$ is indeed constant on $\Pi$.
\epf

Next we study how $g$ behaves as one moves away from $\Pi$. In a vicinity of $\Pi$, this is up to small order terms captured by the nonzero eigenvalues of the Hessian of $g$ on $\Pi$.
Write
\[H(u,v) = g''(\kk(u,v)),\, 0<u<\frac{3}{10},\, \frac35 < v < 1\]
for the Hessian of $g$ at $\kk(u,v)$,
Since $g$ is constant on $\Pi$, it follows that two of the four eigenvalues of $H(u,v)$ are $0$. Write $\lambda_1(u,v)$ and $\lambda_2(u,v)$ for the other two eigenvalues. These are the two roots of a second degree polynomial with coefficients depending on $(u,v)$, i.e.
\begin{align}
  \lambda_1(u,v) & = \frac{\alpha_0(u,v)-\sqrt{\beta_0(u,v)}}{\delta(u,v)} \label{eig1} \\
  \lambda_2(u,v) & = \frac{\alpha_0(u,v)+\sqrt{\beta_0(u,v)}}{\delta(u,v)} \label{eig2},
\end{align}
where $\alpha_0$, $\beta_0$ and $\delta$ are polynomials in $(u,v)$. The precise expressions of these can be found and analyzed using your favorite mathematical software; we used a combination of Matlab and Maple. It turns out that
\[\delta(u,v) = u(\frac{3}{10}-u)(\frac{9}{20}-u)(v-\frac{9}{20})(v-\frac35)(1-v).\]
The expressions for $\alpha_0$ and $\beta_0$ are quite extensive and can be found in an appendix. As can be seen there, $\alpha_0^2$ and $\beta_0$ have a factor $(v-u)^4$ in common, and it will be more convenient to work with $\alpha(u,v)=-10^{-6}\alpha_0(u,v)/(v-u)^2$ and $\beta(u,v)=10^{-12}\beta_0(u,v)/(v-u)^4$, where the constants are chosen so that the leading terms of $\alpha$ and $\beta$ have coefficient $1$.

\begin{proposition} \label{negativeeigenvalues}
  The nonzero eigenvalues $\lambda_1(u,v)$ and $\lambda_2(u,v)$ of $H(u,v)$ are, as functions of $(u,v)$, continuous, negative and bounded away from $0$.
\end{proposition}

\bpf
It is easy to convince oneself of the truth of Proposition \ref{negativeeigenvalues} using some mathematical software. Indeed it seems that $\max(\lambda_1(u,v),\lambda_2(u,v)) < -3$. However a proper proof requires more work. We start with some claims about the limit behavior when approaching the points at which $\delta(u,v)$ vanishes, which are by the expression for $\delta$ all points on the boundary of the domain. For the proof of the proposition, it suffices to prove that $\lambda_i(u,v)$ are negative and stay bounded away from $0$ as one approaches the boundary, Our claims go a bit further, since this comes with no extra cost in the proof. They are as follows. As $(u,v) \ra (0,v_0)$ for some $v_0 \in (3/5,1)$,
\[\lambda_1(u,v) = \left(1+o(u)\right)\frac{C(v_0)}{u} ,\,\, \lambda_2(u,v) \ra C(v_0)\]
where $C$ is a generic notation (i.e.\ not the same function in each instance) for a continuous negative function. Similarly, as $u \ra 3/10$ and $v \ra v_0 \in (3/5,1)$,
\[\lambda_1(u,v) = \left(1+o\left(\frac{3}{10}-u\right)\right)\frac{C(v_0)}{\frac{3}{10}-u} ,\,\, \lambda_2(u,v) \ra C(v_0).\]
Analogously
\[\lambda_1(u,v) = \left(1+o\left(v-\frac35\right)\right)\frac{C(u_0)}{v-\frac35} ,\,\, \lambda_2(u,v) \ra C(v_0)\]
and
\[\lambda_1(u,v) = (1+o(1-v))\frac{C(u_0)}{1-v} ,\,\, \lambda_2(u,v) \ra C(v_0)\]
as $v \ra 3/5$ and $v \ra 1$ respectively and $u \ra u_0 \in (0,3/10)$.

When approaching a corner of the domain, the following holds as $(\mu,\nu) \downarrow (0,0)$,
\begin{align*}
  \lambda_1\left(\mu,\frac{3}{5}+\nu\right) & = (1+o(\mu\nu))\frac{C}{\mu\nu}, \\
  \lambda_2\left(\mu,\frac35+\nu\right) & \ra C,  \\
  \lambda_1(\mu,1-\nu) & = (1+o(\mu\nu))\frac{C}{\max(\mu,\nu)}, \\
  \lambda_2(\mu,1-\nu) & = (1+o(\mu\nu)) \frac{C}{\min(\mu,\nu)},  \\
  \lambda_1\left(\frac{3}{10}-\mu,\frac35+\nu\right) & = (1+o(\mu\nu))\frac{C}{\max(\mu,\nu)}, \\
  \lambda_2\left(\frac{3}{10}-\mu,\frac35+\nu\right) & = (1+o(\mu\nu)) \frac{C}{\min(\mu,\nu)}, \\
  \lambda_1\left(\frac{3}{10}-\mu,1-\nu\right) & = (1+o(\mu\nu))\frac{C}{\mu\nu}, \\
  \lambda_2\left(\frac{3}{10}-\mu,1-\nu\right) & \ra C,
\end{align*}
where $C$ is here a generic notation for a negative constant.
To prove these claims,
note first that $\delta$ is nonnegative in its domain. Next, let us show how to prove that $\alpha_0 \leq 0$ with strict inequality in the interior of its domain, i.e.\ that $\alpha \geq 0$ with strict inequality in the interior.
Write $\gamma(w,z) = \alpha(3w/10,1-2z/5)$, $w,z \in [0,1]$. As $\gamma(0,0)$ and $\gamma(1,1)$ are both zero, we use a Taylor expansion around these points. Using e.g.\ Maple, it is seen that $\gamma'_w(0,0) = 9625/384$ and $\gamma'_z(0,0) = 1625/64$ so that
\[\gamma(w,z) = \frac{9624}{384}w + \frac{1625}{64}z + R\]
where the remainder term is of the form $R=\gamma''_{ww}(w_0,z_0)+2\gamma''_{wz}(w_0,z_0)+\gamma''_{zz}(w_0,z_0)$ for some point $(w_0,z_0)$ on the line between the origin and $(w,z)$. Summing up the absolute values of the terms in the partial second derivatives, we find that these are in absolute value all bounded by $163$ on $[0,1/20] \times [0,1/20]$. Being slightly generous with the linear term coefficients, it follows (since $w \geq w_0 \geq 0$ and $z \geq z_0 \geq 0$) that
\[\gamma(w,z) > 25(w+z)-163(w+z)^2\]
which is larger than $0$ on $[0,1/20]\times [0,1/20]$ except at the origin.
Analogously one shows that $\gamma(w,z)>0$ on $[19/20,1] \times [19/20,1]$ except at $(1,1)$.
Next, a plot in Maple with mesh size $2000 \times 2000$ reveals that on the points $(w,z)$ on the mesh and outside these two squares, $\gamma(w,z)>0.038$. An analogous analysis of the partial first derivatives of $\gamma$ shows that these are bounded by $48$. It follows that the chosen mesh size suffices to draw the desired conclusion.

To prove the negativity of the eigenvalues, it remains to show that $\beta(u,v) \leq \alpha(u,v)^2$ with strict inequality in the interior of the domain. Write $\xi(w,z)=\alpha(3w/10,1-2z/5)^2 - \beta(3w/10,1-2z/5)$, $w,z \in [0,1]$. Then your favorite software reveals that $\xi$ is quite amenable:
\begin{align*}
\xi(w,z) &= \frac{729}{125 \cdot 10^{13}}w(1-w)(3-2w)(2-w)(10-3w) \\
& \cdot z(1-z)(11-8z)(7-4z)(5-2z)h(w,z),
\end{align*}
where
\begin{align*}
h(w,z) &= 1744w^2z^2-4760w^2z-5280wz^2+3475w^2+13800wz+4400z^2 \\
&-9750w-11000z+8125.
\end{align*}
Differentiating with respect to $w$ gives
\[h'_w(w,z)=3488wz^2-9520wz-5280z^2+6950w+13800z-9750\]
and setting $h'_w(w,z)=0$ and solving for $w$ gives the solution
\[r(z)=\frac{2640z^2-6900z+4875}{1744z^2-4760z+3475}.\]
Writing $n_r(z)$ and $t_r(z)$ for the numerator and denominator respectively, we find that $n_r$, $t_r$ and $n_r-t_r$ all have no real roots and that $n_r(0)$, $t_r(0)$ and $n_r(0)-t_r(0)$ are all positive.
It follows that $h'_w(\cdot,z)$ has no zero in $[0,1]$ for an arbitrary fixed $z$. Since $h'_w(0,0)<0$ and $h'_w(0,z)$ has no real roots as a function of $z$, $h(\cdot,z)$ is decreasing. Since $h(1,z)=864z^2-1960z+1850$ has no real roots and $h(1,0)>0$, it finally follows that $h(\cdot,z)$ is strictly positive for all $z$, i.e.\ $h$ is strictly positive for all $w,z \in [0,1]$.
A quick consideration of the product of the other factors in $\phi(w,z)$ now shows that $\phi$ is $0$ on the boundary of $[0,1] \times [0,1]$ and strictly positive in the interior.

\smallskip

For the claimed limit behaviors, consider the case $(u,v) = (w,3/5+z)$ for small $w \in (0,3/10)$ and $z \in (0,2/5)$.
It turns out (this can be seen using the expressions for $\alpha(u,v)$ and $\beta(u,v)$ in the appendix) that
\[\beta\left(w,\frac35+z\right) + \alpha\left(w,\frac35+z\right) = C_1wz+O(w^3+z^3),\]
\[\beta\left(w,\frac35+z\right) - \alpha\left(w,\frac35+z\right) = C_2+O(w+z),\]
and $\alpha(w,\frac35+z) = C_3$ for negative constants $C_1,C_2,C_3$.
This proves the claims for $(u,v) \ra (0,v_0)$ and $(u,v) \ra (u_0,3/5)$ including when $u_0=0$ or $v_0=3/5$ (but neither $v_0=1$ nor $u_0 = 3/10$).

Consider now $(u,v) = (w,1-z)$. We have $\alpha(w,1-z)=C_1w+C_2z+O(w^2+z^2)$ for positive constants $C_1$ and $C_2$ and
\[\beta(w,1-z)-\alpha(w,1-z)^2 = C_3w^2+C_4z^2 + O(w^3+z^3)\]
and
\[\beta(w,1-z)+\alpha(w,1-z)^2 = C_5wz + O(w^3+z^3).\]
This proves the claims for when $(u,v) \ra (u_0,1)$ including $u_0=0$. The remaining claims are analogous.
\epf

As a consequence of these results, it follows that for any given unit vector $\xx$ orthogonal to $\Pi$ at some point $\kk(u,v) \in \Pi$ there is a $C(u,v,\xx) \in [\lambda_1(u,v),\lambda_2(u,v)]$ such that
\[g(\kk(u,v)+t\xx) = g(\kk(u,v))(1-C(u,v,\xx)t^2+O(t^3)),\]
where $C(u,v,\xx)$ is a continuous function of $(u,v,\xx)$ bounded away from $0$.
Hence for any $t=o(1)$,
\[G(\kk(u,v)+t\xx) = G(\kk(u,v))(1+O(t^3))e^{-C(u,v,\xx)t^2}.\]
We strongly believe that $G(\kk(u,v)+t\xx)$ is decreasing in $t$. However the function and its derivative with respect to $t$ seem to be intractable to analyze. Instead we observe that for all $t$, one has at least some $C(u,v,\xx)$ bounded away from $0$ such that $G(\kk(u,v)+t\xx) \leq G(\kk(u,v))e^{-C(u,v,\xx)t^2}$. To see this, observe that since the posterior distribution only depends on $\W$ via the number of word 1 tokens in each document. Denote these by $O_d$ and write also $A_d$ for the number that topic $A$ appears in document $d$. Then the posterior conditioned on $(\phi_A,\phi_B)$ can be computed as follows
\[\Pro(\K=m(a,b,c,d)|O_1=0.3m,O_2=0.6m,\phi_A,\phi_B)\]
\[ = \frac{\Pro(A_1=k_{1.},A_2=k_{2.})\Pro(\K=m(a,b,c,d)|A_1=k_{1.},A_2=k_{2.},\phi_A,\phi_B)}{\Pro(O_1=0.3m,O_2=0.6m,\phi_A,\phi_B)}.\]
Since the $A_d$'s are iid uniform and independent of the $\phi$'s, the first factor of the numerator equals $1/(m+1)^2$. The second factor equals the probability that four independent binomial random variables with parameters $(m(a+b),u)$, $m(1-a-b),v)$, $m(c+d),u)$ and $(m(1-c-d),v)$ respectively equal $ma$, $m(3/10-a)$, $mc$ and $m(3/5-d)$ respectively. Thus
$\Pro(\K=\kk(u,v)|\W=\w,\phi_A,\phi_B) = \Theta(n^{-2})$ for any $(\phi_A,\phi_B) \in [u \pm 1/\sqrt{m}]\times[v \pm 1/\sqrt{m}]$, whereas $\Pro(\K=\kk(u,v)+t\xx) \leq e^{-Cmt^2}$ for any $(u,v)$ and any $\phi_A$ and $\phi_B$. This follows from standard Chernoff bounds. The denominator is of order $1/m^2$ for $(\phi_A,\phi_B) \in [0,3/10] \times [3/5,1]$.

\medskip

Consider now the Markov chain $Z=\{Z_t\}$ on $S=[3m/10] \times [7m/10] \times [3m/5] \times [2m/5]$ defined as follows. Its stationary distribution is given by
\[\pi_Z(\kk) = C G(\kk)^m\]
for a normalizing constant $C$; this differs from $\pi_L$ only in that $\pi_Z$ neglects the low order factors in terms of the function $s$ which only make an essential difference close to $\Pi \cap \partial S$.
The updates are then made by proposing to change the present state $\kk$ to a state $\kk'$ which is chosen randomly among the eight states such that $\n \kk - \kk'\n_1 = 1$ and then making the change with probability $\pi_Z(\kk')/(\pi_Z(\kk')+\pi_Z(\kk))$.

\begin{proposition} \label{mixZ}
  There is a set $B$ of states such that $\pi_Z(B) \geq 1-1/m^{10}$ such that for all $\kappa = \Omega(1/m)$ and all $x \in B$,
  \[\taum^Z(x,\kappa) \leq C(\kappa)m^2.\]
\end{proposition}

From here on, we will use the notation $C$ for a generic positive constant.

\smallskip

\bpf
For $j=1,2,\ldots,2J:=2\lceil 10\log m \rceil$, let $B_j=\{\kk:\pi_Z(\kk) \geq e^{-j}\}$.
Observe that $\pi_Z(B_J)= 1 - n^{-10}$. We also have that
\begin{equation}\label{eq_conductance}
  \Phi_Z(r) \geq C\frac{1}{mr^{1/3}}
\end{equation}
for $r \leq 1/2$ and $\Phi(r) \geq C/m$ for $r > 1/2$ for $\{Z_t\}$ reflected at $\partial B_J$. This coincides with $Z$ itself as long as $Z$ does not visit a vertex neighboring $\partial B_{2L}$. However, starting $Z$ from any vertex $z_0$ within $B_J$, the expected number of visits to any vertex $x \in B_J^c$ is bounded by $m^{-6}$.
This follows from the well known fact of Markov theory that the expected number of visits to $x$ between two consecutive visits to $z_0$ is $\pi_Z(x)/\pi_Z(z_0)$. Hence, by Markov's inequality, with probability $1-m^{-6}$ no such $x$ will be visited.
From this and Theorem \ref{tMP}, it follows that started from $B_J$, $Z$ mixes in time $Cm^2$.

To see that (\ref{eq_conductance}) holds, note that $\pi_Z(B_3) > 19/20$. Since any set $A \subseteq B_J$ with $1/10 \leq \pi_Z(A) \leq 1/2$ must have $1/20 \leq \pi_Z(A \cap B_3) \leq 1/2$, it follows that there is a $C>0$ such that $\Phi_Z(r) \geq C/mr^{1/3}$ for $r \geq 1/10$; this follows on comparing with simple random walk on $B_3$. On the other hand, this is very easily seen to hold true also for $A \subseteq B_J$ with $\pi_Z(A) < 1/10$.
Hence $Z$ has a conductance profile of $C/mr^{1/3}$.
Consequently $Z$ reflected at $\partial B_{2J}$ has a mixing time of order $m^2$ and since $Z$ started within $B_J$ with probability $1-o(1)$ does not deviate from the reflected version within that time, this goes also for $Z$ itself started from within $B_J$.
\epf

\smallskip

Next, we would have liked to modify (\ref{eq_conductance}) to a lower bound $\Phi(r) \geq C\frac{1}{nr^{1/3}}$ also for $L$. However, we have not been able to do this formally. Instead we will use a canonical paths argument for the conductance of $Z$ restricted to $B_J$ that can be readily modified in the desired way.
First observe that the surface $\Pi$ can be alternatively expressed as the function surface $\{\kk(a,d)=(a,b(a,d),c(a,d),d):(a,d) \in [0,3/10] \times [0,2/5]\}$.
Indeed (using your favorite software)
\[b(a,d)=\frac{\left(-10a+3d+3+\sqrt{r(a,d)}\right)\left(4a+17d+3+\sqrt{r(a,d)}\right)}{-176a+162d-6+26\sqrt{r(a,d)}}\]
and
\[c(a,d)=\frac{\left(4a-9d+3-\sqrt{r(a,d)}\right)\left(16a+3d+3-\sqrt{r(a,d)}\right)}{-176a+162d-6+26\sqrt{r(a,d)}}\]
where $r(a,d)=16a^2+a(24-144d)+9(d+1)^2$. One can see that these are continuous with bounded partial derivatives.
By Proposition \ref{negativeeigenvalues}, $\partial B_j$ can be expressed as $h_l(a,d,\omega)$, $(a,d,\omega) \in [0,3/10] \times [0,2/5] \times [0,2\pi)$, where $h_j(a,b,c,d)$ is the distance from $\kk(a,d)$ to $\partial B_j$ in the $bc$-plane in the direction $\omega$ and
\[h_j(a,d,\omega) = \left(1+O\left(\frac{j^{3/2}}{\sqrt{m}}\right)\right)\sqrt{\frac{j}{m}}f(a,d,\omega)\]
for a continuous bounded, and bounded away from 0, function $f$ such that $f(a,d,\cdot)$ describes an ellipse in the $bc$-plane. In particular the intersection of $h_j$ with the $bc$-plane is convex for sufficiently large $m$.
Hence we can choose a canonical path in $B_j$ from $x_1=(a_1,b_1,c_1,d_1)m$ to $x_2 = (a_2,b_2,c_2,d_2)m$, $x_1,x_2 \in B_j$ in the following way. Fix $\alpha \in [0,1]$ and $\omega_2$ such that $x_2 = \alpha h_j(a_2,d_2,\omega_2)$ and write $(a_1,b_0,c_0,d_1)=\alpha h_j(a_1,d_1,\omega_2)$. Start by moving to a nearest neighbor $x$ of $x_1$ for which the $a$ and $d$ coordinates are odd and then
\begin{itemize}
  \item[(i)] walk from $x$ to $(a_1,b_0,c_1,d_1)$ via $\{(a_1,b,c_1,d_1):b \in [b_1,b_0]\}$,
  \item[(ii)] walk from $(a_1,b_0,c_1,d_1)$ to $(a_1,b_0,c_0,d_1)$ via $\{(a_1,b_0,c,d_1):c \in [c_1,c_0]\}$,
  \item[(iii)] walk from $(a_1,b_0,c_0,d_0)$ to $x_2$ via first walking
  \[\{(a,\alpha h_j(a,d_1,\omega_2)\cos \omega_2,\alpha h_j(a,d_1,\omega_2)\sin \omega_2,d_1):a \in [a_1,a_2]\}\]
  and then walking
  \[\{(a_2,\alpha h_j(a_2,d,\omega_2)\cos \omega_2,\alpha h_j(a_2,d,\omega_2)\sin \omega_2,d),d:d \in [d_1,d_2]\}.\]
\end{itemize}
Of course, by the discreteness of the underlying lattice, the paths of step (iii) are chosen as close as possible to the continuous paths indicated, with the extra condition that edges in the $b$- and $c$-directions are used only when their $a$- and $d$-coordinates are even.
This condition together with the starting step makes sure that no edge is ever used in more than one of (i), (ii) or (iii), with the exception that some $a$- and $d$-edges will also be used in some paths starting in their neighbors.

Then, as the volume of $B_j$ grows linearly in $j$, it is easy to see that there is a constant $C$ such that for each $e=\{x,y\} \in B_j \setminus B_{j-1}$,
\[Q(e) \geq \frac{1}{Ce^{j}m^3}.\]
Since no canonical path between a pair of vertices which are both in a given $B_j$ uses any edge in $B_j^c$, so that any canonical path using an edge $e \in B_j \setminus B_{j-1}$ has all vertices on at least one side of $e$ in $B_{j-1}^c$,
\begin{equation}\label{econg}
  \sum_{(x,y):e \in \gamma_{x,y}}\pi(x)\pi(y) \leq \frac{C}{e^{j}m^2}.
\end{equation}
To spell this out, assume for simplicity that $e \in B_1$. Write $\pi_{before}$ and $\pi_{after}$ for the total respective stationary probability masses of vertices that can appear before or after $e$ respectively in the above algorithm for finding canonical paths. Then we have, using a modification of a standard argument for random walk on $\Z_m^4$, that
\begin{itemize}
  \item if $e$ appears in (i), then $\pi_{before} \leq C/m^{5/2}$ and $\pi_{after} \leq 1$ and so the sum in (\ref{econg}), is bounded by $\pi_{before}\pi_{after} \leq C/m^{5/2}$.
  \item if $e$ appears in (ii), then $\pi_{before} \leq C/m^{2}$ and $\pi_{after} \leq C/\sqrt{m}$ and so the sum in (\ref{econg}) is bounded by $C/m^{5/2}$.
  \item if $e$ appears when adjusting $a$ in (iii), then $\pi_{before} \leq C/m$ and $\pi_{after} \leq C/m$ and so the sum in (\ref{econg}) is bounded by $C/m^2$.
  \item if $e$ appears when adjusting $b$ in (iii), then $\pi_{before} \leq 1$ and $\pi_{after} \leq C/m^2$ and so the sum in (\ref{econg}) is bounded by $C/m^2$.
\end{itemize}
Adjusting the argument for $e$ in a general $B_j$ is easy.
Summing up, we get $\rho(\Gamma) \leq C m$ and hence
\[\taum^{Z}(x,\kappa) \leq C m^2 \log(2m+\kappa^{-1})\]
for all $x \in B_L$ by Theorem \ref{tSJ}.

\smallskip

Next we observe that $Z$ and $L$ differ in two important ways. One of them is that their stationary distributions are different. The other is that $L$ spends a very long time in some states whereas $Z$ does not. Consider e.g.\ when $L_t = m(a,b,c,d)$, where $a=0$. Then by studying $\pi_R$, we see that a proposed change from an $A$ to a $B$ of a topic behind a word 1 in document 1 is accepted only with probability $1/(0.3m)$ and hence (since such a change is proposed with probability $0.3/2$) the time taken to change $a$ in the lumped chain will be on average $2m$. For $Z$ however, such a change will always take place with probability at least (say) $1/20$.
In general, the probability to make a change in the $a$-direction for $L$ is $C(a+1/m)(3/10-a+1/m)$ and analogously for the other directions, where $C$ is of constant order.

\smallskip

We will now adjust the arguments for $Z$ to $L$ in two steps, where first step takes care of the latter difference.
Consider any Markov chain $X=\{X_t\}$ on $B_J$ with the same stationary distribution as $Z$ for which there are constants $C_1,C_2$ such that $\Pro(X_1=m(a \pm 1/m,b,c,d)|X_0=m(a,b,c,d)) \in [C_1(a+1/m)(3/10-a+1/m),C_2(a+1/m)(3/10-a+1/m)]$ and analogously for $b$, $c$ and $d$. Use the same canonical paths as before.
Consider any edge $e=(m(a,b,c,d),m(a,b+1/m,c,d)) \in B_1$. Unlike for $\{Z_t\}$, $Q(e)$ is now $C(b+1/m)(7/10-b+1/m)/m^3$ rather than $C/m^3$.
However, by the nature of the canonical paths algorithm given, a closer look at (i) reveals that for a path from $x_1$ to $x_2$, $x_i=(a_i,b_i,c_i,d_i)$, to use $e$, either $b_1 \leq mb \leq b_2$ or vice versa. In the former case $e$ will be traversed from left to right (in the $b$ direction) and in the latter it will be traversed from right to left by the canonical path. For the union of all paths traversing $e$ from left to right $\pi_{before} \leq C(b+1/m)/m^{5/2}$ and $\pi_{after} \leq C(7/10-b+1/m)$. This compensates exactly for the extra factor $b$ in $Q(e)$ and hence the conclusion that $\rho(e) \leq C\sqrt{m}$ remains. The union paths traversing $b$ from right to left gets $\pi_{before} \leq C(7/10-b+1/m)/m^{5/2}$ and $\pi_{after} \leq C(b+1/m)$. Steps (ii) and (iii) are analogous as is the adjustment for general $B_j$. Hence
\[\taum^{X}(x,\kappa) \leq C m^2 \log(2m+\kappa^{-1})\]
for all $x \in B_J$.

\smallskip

Next consider $L$. Obviously $L$ will also with probability $1-o(1)$ not hit the boundary of $B_J$ in the time frame under consideration, so we may consider $L$ restricted to $B_J$.
Again use the same canonical paths as for $Z$. Let $X$ be defined as the lumped Markov chain of the Gibbs sampler $U=\{U_t\}$ on $\{A,B\}^{2m}$ driven by the stationary distribution
\[\pi_U(\z) = \frac{1}{\binom{\frac{3m}{10}}{am}\binom{\frac{7m}{10}}{bm}\binom{\frac{3m}{5}}{cm}\binom{\frac{2m}{5}}{dm}}G(a,b,c,d)^m\]
for $m(a,b,c,d)=\kk(\z)$. Then $\{X_t\}$ is a Markov chain of the kind just considered and it is such that the difference between $L$ and $X$ is that the weights corresponding to $L$ as compared to those of $X$ are altered so that the stationary probability $\pi_L(x)$ of vertex $x=(am,bm,cm,dm)$ is $Cs(a,b,c,d)\pi_Z(x)$, where
\[s(a,b,c,d) = \frac{s\left(\frac{10(a+c)}{9}\right)s\left(\frac{10(b+d)}{11}\right)s(a+b)s(c+d)}
{s\left(\frac{10a}{3}\right)s\left(\frac{10b}{7}\right)s\left(\frac{5c}{3}\right)s\left(\frac{5d}{2}\right)}\]
and $C$ is of constant order (note that $a+b+c+d$ is of constant order on $B_{2J}$.
Here, for $e = (m(a,b,c,d),m(a,b+1/m,c,d)) \in B_1$,
\[Q(e) = C(b+1/m)(7/10-b+1/m)\pi_L(x).\]
In (i), we get with $K$ being the diameter of the intersection of $B_1$ with the $bc$-plane, that for the paths traversing $b$ from left to right
\[\pi_{before} \leq C(b+1/m)\pi_L(x)\sum_{i=1}^{K}\frac{\sqrt{K}}{\sqrt{i}} = CK(b+1/m)\pi_L(x) \leq Cb\sqrt{m}\pi_L(x)\]
and still $\pi_{after} \leq C(7/10-b+1/m)$ and similarly for paths going from right to left. Hence $\rho(e) \leq \sqrt{m}$.
Again parts (ii) and (iii) are analogous as is the adjustment for general $B_j$ is.
Summing up, we have
\begin{proposition} \label{mixL}
  There is a set $B$ of states such that $\pi_Z(B) \geq 1-1/m^{10}$ such that for all $\kappa = \Omega(1/m)$ and all $x \in B$,
  \[\taum^L(x,\kappa) \leq C(\kappa)m^2\log m.\]
\end{proposition}

\smallskip

Finally we take the step from the lumped Markov chain $L$ to the original Gibbs sampler $R$. Let $\kappa = \Omega(1/m)$, pick $C_0$ sufficiently large that $\n \Pro(L_t \in \cdot)-\pi_L \n_{TV} \leq \kappa\}$ whenever $t \geq t_0 := (C_0/2)m^2\log m$. Start a Gibbs sampler $R^0$ from $\pi_R$ and let $L^0$ be the corresponding lumped Markov chain. It is then possible to design a coupling of $L$ and $L^0$ such that $\Pro(L_{t_0}=L^0_{t_0}) \geq 1-\kappa/2$. Run this coupling up to time $t_0$. Then if $L$ and $L_0$ do not agree at that time, let the updates be dependent in an arbitrary way, e.g.\ simply independent until they meet.
Else, $L_{t_0}=L^0_{t_0}$, which means that the underlying Gibbs samplers $R$ and $R_0$ agree in the number of $A$'s in document $d$ at word $j$ for $d,j = 1,2$.

Now couple the updates from time $t_0$ and on in the following way. First pair up the positions in the corpus by pairing each position at which $R_t$ and $R^0_t$ agree with itself and pairing each token $(i,j)$ for which $R_t(i,j)=x$ and $R^0_t(i,j)\neq x$ with a token $(i^0,j^0)$ in the same document with the same word, for which $R^0_t(i^0,j^0) \neq x$ and $R_t(i^0,j^0)=x$, $x=A,B$. Clearly this can be done so that each position is paired with one and only one position. Pick such a pairing.

Next, to pick $R_{t+1}$ and $R^0_{t+1}$, pick for $R_0$ the position $(i^0,j^0)$ that is paired with the position $(i,j)$ chosen for $R$. The key observation now is that the conditional distribution of $R^0(i^0,j^0)$ given $R=(i',j')$, $(i',j') \neq (i^0,j^0)$ is the same as that for $R(i,j)$ given $R(i',j')$, $(i',j') \neq (i,j)$. Hence we may, and do, couple the updates so that $R$ and $R^0$ still agree at time $t+1$. This means that no new disagreements between $R$ and $R^0$ will ever agree and whenever a position $(i,j)$ at which $R$ and $R^0$ disagree is updated, there is a probability that they agree there, and then also at $(i^0,j^0)$, after the update. This probability might be quite small. However, there are two ways that a new agreement at $(i,j)$ can come about, either by picking $(i,j)$ or by picking $(i^0,j^0)$ to be updated in $R$, and at least one of these choices gives a probability of at least $1/2$ of agreement after the update (namely the choice that proposes to change a topic that is currently in minority). Hence, if there are $D_t$ positions of disagreement at time $t \geq t_0$, there is a probability of at least $((D_t/2)/(2m)) \cdot (1/2)$ of reducing the number of disagreements by two for one time step. Since there are at most $2m$ disagreements at time $t_0$, the coupling time $T$, satisfies that $T-t_0$ is stochastically dominated by the sum $\sum_{j=1}^{m}\xi_j$, where the $\xi_j$:s are independent and geometric with parameters $(m-2j)/(4m)$ respectively. Hence $\E[T-t_0] \leq (1+o(1))2m\log m$ and $\Var(T-t_0)<7m^2$ and so by Chebyshev's inequality $\Pro(T-t_0 > (C_0/2)m^2\log m) = O(1/(m^2\log^2m)) < \kappa/2$ for sufficiently large $m$.

Combining the two steps of the coupling argument, it follows that $\Pro(T>C_0m^2\log m) < \kappa$ for sufficiently large $m$. Now Theorem \ref{ta} follows from the coupling inequality.

\medskip

{\bf Final remarks.} Needless to say the situation considered here, with only two unique words in the corpus, is very unrealistic, as is the fact that we use only two documents. Using only two topics is of course also a bit unrealistic, but since the number of topics is typically very limited compared to the size of the corpus, this is not as serious.
However using more words and more documents does not seem to impose any fundamental differences in terms of the analysis here (even though it may become intractable) as long as each word appears a large number of times in each document. Indeed it is easy to generalize Proposition \ref{piR} to the case with $D$ documents and $v$ distinct words, using the general statement of Lemma \ref{lb}. One gets
\[\pi_R(\z) = C\frac{ \binom{n_{..}}{k_{..}} }{(k_{..}+1)(n_{..}-k_{..}+1)\,\prod_{d=1}^{D}\binom{n_{d.}}{k_{d.}}\,\prod_{j=1}^{v}\binom{n_{.j}}{k_{.j}}}\]
and for the corresponding lumped Markov chain
\[\pi_L(\kk) = C \frac{ \binom{n_{..}}{k_{..}} \prod_{d=1}^{D} \prod_{j=1}^{v} \binom{n_{dj}}{k_{dj}} }{(k_{..}+1)(n_{..}-k_{..}+1)\,\prod_{d=1}^{D}\binom{n_{d.}}{k_{d.}}\,\prod_{j=1}^{v}\binom{n_{.j}}{k_{.j}}}\]
and when all $n_{dj}$:s are larger than logarithmic in $Dm$ (having all documents of large equal length $m$), an analogous treatment seems to be in principle possible.
This condition on the $n_{dj}$'s is of course also not true in most practical situations. What happens in such situations is an important question worthy of further study.

Another assumption we made, was that the topic distribution per document prior $\alpha$ and the word distribution per topic prior $\beta$ were both $(1,1)$. It is common to use symmetric priors but usually with smaller parameters. In case of $\beta$ this does not seem to be likely to make much of a difference as long as we keep the number of words very low compared to the length of the documents.
For $\alpha$, the value used will typically depend on the context and may be optimized through cross validation. In any case, I conjecture (or rather speculate) that with $\alpha=(\ep,\ep)$ for a small constant $\ep$ independent of $m$, the mixing time for $\ep<1/2$ is
of order $m^{3-2\ep}\log m$.

\section{Appendix}

The functions $\alpha$ and $\beta$ are given by
\begin{multline*}
\alpha_0(u,v) = \frac{25}{2}(v-u)^2 \Big( 80000\,{u}^{4}{v}^{4}-152000\,{u}^{4}{v}^{3}-152000\,{u}^{3}{v}^{4}+
90000\,{u}^{4}{v}^{2} \\
+252400\,{u}^{3}{v}^{3}+90000\,{u}^{2}{v}^{4}-
18000\,{u}^{4}v-122220\,{u}^{3}{v}^{2}-122220\,{u}^{2}{v}^{3} \\
-18000\,u{v}^{4}+12420\,{u}^{3}v+27738\,{u}^{2}{v}^{2}+12420\,u{v}^{3}+3240\,{u
}^{3} \\
+16182\,{u}^{2}v +16182\,u{v}^{2}+3240\,{v}^{3}-6156\,{u}^{2} \\
- 15633\,vu-6156\,{v}^{2}+3645\,u+3645\,v-729 \Big).
\end{multline*}

\begin{multline*}
\beta_0(u,v) = \Big(\frac{25}{2}\Big)^2(v-u)^4 \Big( 6400000000\,{u}^{8}{v}^{8}-24320000000\,{u}^{8}{v}^{7} \\
-24320000000\,{u
}^{7}{v}^{8}+37504000000\,{u}^{8}{v}^{6}+93568000000\,{u}^{7}{v}^{7} \\
+37504000000\,{u}^{6}{v}^{8}-30240000000\,{u}^{8}{v}^{5}-146374400000\,
{u}^{7}{v}^{6}-146374400000\,{u}^{6}{v}^{7} \\
-30240000000\,{u}^{5}{v}^{8
}+13572000000\,{u}^{8}{v}^{4}+120018240000\,{u}^{7}{v}^{5}+
233005760000\,{u}^{6}{v}^{6} \\
+120018240000\,{u}^{5}{v}^{7} +13572000000\,{u}^{4}{v}^{8}-3240000000\,{u}^{8}{v}^{3}-54971856000\,{u}^{7}{v}^{4
} \\
-195046704000\,{u}^{6}{v}^{5}-195046704000\,{u}^{5}{v}^{6}-
54971856000\,{u}^{4}{v}^{7}-3240000000\,{u}^{3}{v}^{8} \\
+324000000\,{u}^
{8}{v}^{2}+13500432000\,{u}^{7}{v}^{3}+91833066000\,{u}^{6}{v}^{4}+
167842288800\,{u}^{5}{v}^{5} \\
+91833066000\,{u}^{4}{v}^{6}+13500432000\,
{u}^{3}{v}^{7}+324000000\,{u}^{2}{v}^{8}-1432080000\,{u}^{7}{v}^{2} \\
-23668200000\,{u}^{6}{v}^{3}-82693612080\,{u}^{5}{v}^{4}-82693612080\,{
u}^{4}{v}^{5}-23668200000\,{u}^{3}{v}^{6} \\
-1432080000\,{u}^{2}{v}^{7}+
11664000\,{u}^{7}v+2901646800\,{u}^{6}{v}^{2}+23482334160\,{u}^{5}{v}^
{3} \\
+44511382260\,{u}^{4}{v}^{4}+23482334160\,{u}^{3}{v}^{5}+2901646800
\,{u}^{2}{v}^{6}+11664000\,u{v}^{7} \\
-127720800\,{u}^{6}v-3799373040\,{u
}^{5}{v}^{2}-15209217720\,{u}^{4}{v}^{3}-15209217720\,{u}^{3}{v}^{4} \\
-3799373040\,{u}^{2}{v}^{5}-127720800\,u{v}^{6}+10497600\,{u}^{6}+
408414960\,{u}^{5}v \\
+3551084388\,{u}^{4}{v}^{2}+6965123256\,{u}^{3}{v}^
{3}+3551084388\,{u}^{2}{v}^{4}+408414960\,u{v}^{5} \\
+10497600\,{v}^{6}-
39890880\,{u}^{5}-606551328\,{u}^{4}v-2281385088\,{u}^{3}{v}^{2} \\
-2281385088\,{u}^{2}{v}^{3}-606551328\,u{v}^{4}-39890880\,{v}^{5}+
61515936\,{u}^{4} \\
+485146584\,{u}^{3}v+933840981\,{u}^{2}{v}^{2}+
485146584\,u{v}^{3}+61515936\,{v}^{4} \\
-49601160\,{u}^{3}-218074518\,{u}
^{2}v-218074518\,u{v}^{2}-49601160\,{v}^{3} \\
+22261473\,{u}^{2}+52435512
\,vu+22261473\,{v}^{2} \\
-5314410\,u-5314410\,v+531441 \Big)
\end{multline*}


\end{document}